\documentclass{article}
\usepackage{ijcai26}

\usepackage{times}
\usepackage{soul}
\usepackage{url}
\usepackage[hidelinks]{hyperref}
\usepackage[utf8]{inputenc}
\usepackage[small]{caption}
\usepackage{graphicx}
\usepackage{amsmath}
\usepackage{amsthm}
\usepackage{booktabs}
\usepackage{algorithm}
\usepackage{algorithmic}
\usepackage{tikz}
\usetikzlibrary{arrows.meta, positioning}
\usepackage[switch]{lineno}

\usepackage{subfig}
\usepackage{dsfont}
\usepackage{todonotes}
\usepackage{thm-restate} 

\usepackage{amsmath,nccmath}
\usepackage{amssymb}
\usepackage{xspace}
\usepackage[dvipsnames,table]{xcolor}
\usepackage{cleveref}

\newcommand{\tiff}{\text{ iff }}
\newcommand{\LTLf}{LTL$_f$\xspace}

\newcommand{\LTL}{{LTL}\xspace}
\newcommand{\trace}{\pi}
\newcommand{\last}{\mathtt{lst}}

\newcommand{\ttrue}{\mathit{tt}}
\newcommand{\ffalse}{\mathit{ff}}
\newcommand{\Eventually}{\mathsf{F}} 
\newcommand{\Always}{\mathsf{G}} 
\newcommand{\WeakNext}{\mathsf{X}} 
\newcommand{\StrongNext}{\mathsf{X^{!}}} 
\newcommand{\Until}{\mathbin{\mathsf{U}}} 

\newcommand{\A}{\mathcal{A}}

\newcommand{\G}{\mathcal{G}}

\newcommand{\M}{\mathcal{M}}

\renewcommand{\P}{\mathcal{P}}
\newcommand{\Q}{\mathcal{Q}} 

\renewcommand{\S}{\mathcal{S}}

\newcommand{\V}{\mathcal{V}}
 
\newcommand{\X}{\mathcal{X}}
\newcommand{\Y}{\mathcal{Y}} 
\newcommand{\Z}{\mathcal{Z}}

\newcommand{\run}{{\rm run}}

\newcommand{\Win}{{\rm Win}}

\newcommand{\maxvalue}{{max-observation}\xspace}
\newcommand{\Maxvalue}{{Max-observation}\xspace}
\newcommand{\MaxValue}{{Max-Observation}\xspace}
\newcommand{\maxg}{{max-guarantee}\xspace}

\newcommand{\MaxG}{{Max-Guarantee}\xspace}

\newcommand{\tool}{\textsc{OptLtlfSynt}\xspace}

\newtheorem{definition}{Definition}

\newtheorem{example}{Example}

\newtheorem{theorem}{Theorem}
\makeatletter
\let\c@theorem\c@lemma
\makeatother

\title{Optimal \LTLf Synthesis}

\author{
Yujian Cao\and
Sven Schewe\and
Qiyi Tang\and
Shufang Zhu\\
\affiliations
University of Liverpool
\emails
\{yujian.cao, sven.schewe, qiyi.tang, shufang.zhu\}@liverpool.ac.uk
}

\newcommand{\qiyi}[1]{\textsf{\textcolor{Orchid!90}{Qiyi: #1}}\marginpar{$\star$}}

\begin{document}\sloppy

\maketitle

\begin{abstract}
Strategy synthesis typically follows an all-or-nothing paradigm, returning unrealisable whenever a specification cannot be guaranteed in an uncertain environment. In this paper, we introduce optimal \LTLf synthesis, where the goal is to realise as many objectives as possible from a given specification consisting of multiple objectives, especially for the case that they are not all jointly realisable. We first consider max-guarantee synthesis, which commits to a maximal set of objectives that we can a priori guarantee to realise. We then introduce max-observation synthesis, which maximises a posteriori realised objectives that may be incomparable on different executions. Finally, we present incremental max-observation synthesis, which further improves strategies by exploiting opportunities for stronger guarantees when they arise during an execution.
Experimental results show that different variations of optimal synthesis scale broadly equally well, solving a large fraction of the benchmark instances within the given timeout, demonstrating the practical feasibility of the approach.
\end{abstract}
\section{Introduction}

Strategy synthesis is a key problem in Artificial Intelligence~(AI), where an autonomous agent must compute a strategy/policy/controller that guarantees task achievement despite uncertainty and nondeterminism in the environment~\cite{Reiter2001,GhallabNauTraverso2016}. Automated synthesis from high-level temporal specifications addresses this challenge by allowing designers to describe \emph{what} the agent should achieve, while the synthesis process determines \emph{how} to do so.
This is formalised by reactive synthesis~\cite{PnueliR89}, where the agent interacts with an adversarial environment and must ensure satisfaction of the specification under all possible environment behaviours. 

For finite and terminating tasks, Linear Temporal Logic over finite traces (\LTLf)~\cite{DegVa13} provides a natural specification language, and \LTLf synthesis~\cite{DegVa15} constructs strategies that guarantee satisfaction of an \LTLf formula along every possible execution.
\LTLf synthesis shares deep similarities with planning in fully observable nondeterministic domains (FOND, strong plans)~\cite{Cimatti03,GeffnerBonet2013}.

Standard approaches to \LTLf synthesis ~\cite{PnueliR89,DegVa15,ZhuTLPV17,DeGiacomoR18} assume that an agent must satisfy a single \LTLf formula while interacting with a fully adversarial environment~\cite{AGMR19,CamachoBM19}.
Under this assumption, synthesis is typically an all-or-nothing approach:
either a strategy that guarantees satisfaction of the specification exists,
or the specification is declared unrealisable and no strategy is returned.
While this model offers strong correctness guarantees, it is often too restrictive in practice, since real environments are rarely fully adversarial~\cite{ADR21,GhallabNauTraverso2016} and users may struggle to capture all desired objectives in a single specification~\cite{DimitrovaGT20,DPZ25}.

To overcome these limitations, prior work has mainly explored two directions. One direction extends the specification formalism to allow multiple objectives, as in preferences-based planning, such as planning with soft goals~\cite{son2006planning,KR06BienvenuM,jorge2008planning,keyder2009soft} and maximum realisability for temporal logic specifications~\cite{DimitrovaGT20}, which focuses on safety properties treated as soft constraints and allows their temporary violation. 
Another direction relaxes assumptions about the environment:
best-effort synthesis~\cite{ADR21,DPZ25} or strong cyclic planning~\cite{Cimatti03,GeffnerBonet2013}, for example, allow agents to exploit cooperative environment behaviours instead of giving up when strict guarantees cannot be enforced.

In this paper, we focus on \LTLf synthesis with multiple objectives and study the problem of \emph{optimal synthesis}, especially when the conjunction of all objectives is not realisable. Here, the synthesis problem is no longer whether a strategy that guarantees all specifications exists, but to find a strategy that guarantees as many of the specifications as possible.

But what makes a strategy optimal? That depends on the circumstances.
Consider an agent where we have five goals of equal value, e.g.\ that a robot eventually visits Rooms 1 through~5. The robot first chooses either a path that allows it to \emph{surely} visit Room 1, or to go down a path where, depending on a choice of the environment, it can guarantee \emph{either} to visit Rooms 2 and 3 \emph{or} to visit Rooms 4 and 5.

If we only value guarantees that we can give \emph{a priori}, then we would favour being sure to visit Room 1.
We refer to this as \emph{max-guarantee synthesis}:
max-guarantee synthesis provides the maximal (or maximally valued) subset of objectives that can be jointly realised in the classic sense.

If we want to maximise (the value of) the specifications that we will achieve, on the other hand, we would prefer to satisfy two goals, even if we cannot commit to any of them in particular. 
We introduce \emph{max-observation} synthesis to address this problem.
Instead of committing to a fixed subset of objectives in advance, max-observation synthesis evaluates a strategy based on the objectives it actually satisfies along each execution, and focuses on maximising the number (or value) of achieved objectives.

In the robot example, we would steer the robot towards the point where it can ensure visiting two rooms, albeit which ones is down to the environment.
This strategy exploits observed environment behaviour, making max-observation synthesis always at least as good as max-guarantee synthesis. Moreover, our framework naturally supports quantitative weights on objectives, enabling users to express the relative importance of individual objectives and allowing the synthesis of strategies that optimise the total achieved value.

We further improve max-observation synthesis to \emph{incremental} max-observation synthesis, where we want to maximise the ensured number~(value) of objectives for every history.
In the robot example, the robot might find itself in a situation where, after visiting Room 2, it can choose to continue to Room 3, or go down a route where it can visit Rooms 4 and 5.
If we only consider what the strategy can ensure from the beginning of its execution, these routes are of equal value, because \emph{other} runs of the robot curtail this value to 2.
In incremental max-observation synthesis, we want more: after each observed history, we seek a strategy that maximises the value that can be ensured from that point onward.
In this setting, we would 
prefer visiting Rooms 4 and 5 over visiting Room 3 only, since---for the given history---this allows us to increase our promise to visit at least three rooms overall.

Both optimisation goals we introduce, max-guarantee and max-observation, are natural and have their place. Some goals are only worth satisfying if one can guarantee them a priori, while in other cases it is just as good to see them satisfied a posteriori. 
We therefore study both and show that they can easily be combined.

We have implemented optimal \LTLf synthesis using symbolic techniques~\cite{ZhuTLPV17,DuretLutzZPGV25}, including max-guarantee synthesis, max-observation synthesis, and incremental max-observation synthesis.
We also evaluated these approaches on benchmarks taken from the \LTLf synthesis track of the annual reactive synthesis competition SyntCOMP (available at \url{https://www.syntcomp.org/}), for which we have constructed our own multi-objective specifications.
Our empirical results show that the different approaches exhibit broadly comparable performance and solve a large fraction of the benchmarks, producing more informative strategies than simply concluding ``unrealisable".
Incremental max-observation synthesis shows similar, sometimes better, performance than non-incremental max-observation, demonstrating that strictly better guarantees can be obtained in practice at little additional cost. 

\section{Preliminaries~\label{sec:preliminaries}}

A \emph{trace} over an alphabet of symbols $\Sigma$ is a finite or infinite sequence of elements from $\Sigma$. The empty trace is $\epsilon$. Traces are indexed starting at zero, and we write $\trace = \trace_0 \trace_1 \cdots$.  The length of a trace is $|\trace|$. For a finite trace $\trace$, we denote by $\last(\trace)$ the index of the last element of $\trace$, i.e.\ $|\trace| - 1$. For $i \leq \last({\trace})$, we have that $\trace_i \in 2^{AP}$ is the $i$-th interpretation of $\trace$. By $\trace^k = \trace_0 \cdots \trace_k$ we denote the \emph{prefix} of $\trace$ up to the $k$-th iteration.

\subsection{\texorpdfstring{\LTLf\ Basics}{LTLf Basics}}
Linear Temporal Logic on finite traces (\LTLf) is a specification language to express temporal properties for finite non-empty
traces~\cite{DegVa13}. \LTLf shares the syntax with \LTL~\cite{Pnu77}. Given a set of atoms $AP$, the \LTLf formulas over $AP$ are generated as follows: 
\[\varphi ::= p \mid \varphi_1 \wedge \varphi_2 \mid \neg \varphi \mid  
 \StrongNext \varphi \mid \varphi_1 \Until \varphi_2.\]
$p \in AP$ is an \textit{atom}, $\StrongNext$~(\emph{Next}), and $\Until$~(\emph{Until}) are temporal operators. 
We use standard Boolean abbreviations such as $\vee$~(or) and $\supset$~(implies), $\ttrue$~(\emph{true}) and $\ffalse$~(\emph{false}). Moreover, we define the following abbreviations: \emph{Weak Next} $\WeakNext \varphi \equiv \neg \StrongNext \neg \varphi$, \emph{Eventually} $\Eventually \varphi \equiv \ttrue \Until \varphi$ and \emph{Always} $\Always \varphi \equiv \neg \Eventually \neg \varphi$. The size of $\varphi$, written $|\varphi|$, is the number of 
its subformulas. We assume standard semantics from~\cite{DegVa13}.

\LTLf formulas are interpreted over finite non-empty traces $\trace$ over the alphabet $\Sigma = 2^{AP}$. For $i \leq \last({\trace})$, we have that An \LTLf formula $\varphi$ \emph{holds} at instant $i$ of a trace $\pi$ is defined inductively on the structure of $\varphi$ as:
\begin{itemize}
	\item 
	$\trace, i \models p \tiff p \in \trace_i$;
	\item 
	$\trace, i \models \lnot \varphi \tiff \trace, i \not\models \varphi\nonumber$;
	\item 
	$\trace, i \models \varphi_1 \wedge \varphi_2 \tiff \trace, i \models \varphi_1 \text{ and } \trace, i \models \varphi_2$;
	\item 
	$\pi, i \models \StrongNext \varphi \tiff  i< \last(\pi)$ and $\trace,i+1 \models \varphi$;
	\item 
	$\pi, i \models \varphi_1 \Until \varphi_2$ iff $\exists j$ such that $i \leq j \leq \last(\pi)$ and $\pi,j \models\varphi_2$, and $\forall k, i\le k < j$ we have that $\pi, k \models \varphi_1$.
\end{itemize}

We say that $\pi$ \emph{satisfies} $\varphi$, written as $\pi \models \varphi$, if $\pi, 0 \models \varphi$. An \LTLf formula $\varphi$ can be represented by a Deterministic Finite Automaton~(DFA) $\A_\varphi = (2^{AP}, \Q, q_0, \delta, F)$, where $2^{AP}$ is a finite alphabet, $\Q$ is a finite set of states, $q_0 \in \Q$ is the initial state, $\delta : \Q \times 2^{AP} \rightarrow \Q$ is the transition function, and $F \subseteq \Q$ is the set of accepting states, such that, for every trace $\pi$, we have $\pi\models \varphi$ iff $\pi$ is accepted by $\A_\varphi$~\cite{DegVa13}.

\subsection{\texorpdfstring{\LTLf\ Synthesis}{LTLf Synthesis}}
An \LTLf synthesis specification can be described by a tuple $(\varphi,\mathcal{X}, \mathcal{Y})$, where $\varphi$ is an \LTLf formula over the propositions in $\mathcal{X \cup Y}$. $\mathcal{X}$ and $\mathcal{Y}$ are two disjoint sets: $\mathcal{X}$ is the set of input variables controlled by the environment and $\mathcal{Y}$ is the set of output variables controlled by the agent. A strategy is a function $\sigma :(2^\mathcal{X})^*\rightarrow2^\mathcal{Y}$ that maps each finite sequence of input variables to an output assignment. $\mathbf{X} \in (2^{\mathcal{X}})^\omega$ denotes an infinite sequence of valuations of input variables, and the trace $\pi(\mathbf{X},\sigma)$ induced by the prefix of $\mathbf{X}$ and $\sigma$ is:
\[
\pi(\mathbf{X}, \sigma) = (X_0 \cup \sigma(\epsilon))(X_1 \cup \sigma(X_0)) \cdots.
\]

\begin{definition}[Winning strategy]
Given an \LTLf synthesis specification $(\varphi,\mathcal{X}, \mathcal{Y})$, a strategy $\sigma:(2^\mathcal{X})^*\rightarrow2^\mathcal{Y}$ is a \emph{winning strategy} for $\varphi$, denoted $\sigma \triangleright \varphi$, if for every infinite sequence of the input variables $\mathbf{X} \in (2^{\mathcal{X}})^\omega$, the induced trace $\trace(\mathbf{X}, \sigma)$ satisfies $\varphi$, i.e.\ there exists a finite prefix of $\trace(\mathbf{X},\sigma)$ that satisfies $\varphi$.
\end{definition}

\begin{definition}[Realisability]
An $\text{LTL}_f$ synthesis specification $(\varphi, \mathcal{X}, \mathcal{Y})$ 
is \emph{realisable} if there exists a winning strategy $\sigma: (2^\mathcal{X})^*\rightarrow2^\mathcal{Y}$ for $\varphi$, i.e.\ $\sigma \triangleright \varphi$. The $\text{LTL}_f$ realisability problem is to determine whether $(\varphi, \mathcal{X}, \mathcal{Y})$ is \emph{realisable}.
\end{definition}

\begin{definition}[Synthesis]
The synthesis problem consists in computing a winning strategy for a realisable specification.
\end{definition}

Sometimes, for simplicity, we omit $\X$ and $\Y$ when they are clear from the context.

\subsection{DFA Games}
Classic solutions to \LTLf synthesis rely on a reduction to two-player DFA games (reachability games).
A \emph{DFA game} is described as a tuple $\G = (\A, F)$, where $\A = (2^{AP}, \Q, q_0, \delta, -)$ is a DFA representing the game arena of which the set of accepting states does not matter and $F \subseteq \Q$ is the \emph{reachability winning objective}. 
For \LTLf synthesis, we often take the corresponding DFA of the formula as the game arena and the set of accepting states as the winning objective. 
Given a play $\pi = (X_0 \cup Y_0)(X_1 \cup Y_1) \ldots \in (2^{\X \cup \Y})^\omega$, running $\A$ on $\pi$ gives us the infinite sequence $\rho= q_0 q_1 \ldots \in \Q^\omega$ such that $q_0$ is the initial state and $q_{i+1} = \delta(q_i, X_i \cup Y_i)$ for all $i \geq 0$.
Since the transitions in $\A$ are all deterministic, we thus denote by $\rho = \run(\pi, \A)$ the unique sequence of running $\A$ on $\pi$.
Analogously, we denote by $\rho^k = \run(\pi^k, \A)$ the unique finite sequence of running $\pi^k$ on $\A$, and $\rho^k = q_0q_1\ldots q_k$.
The \emph{reachability winning objective} $F$ indicates a set of desired states to reach for the agent. A play $\pi$ is a \emph{winning play} with respect to winning objective $F$ if some state $s \in F$ occurs in $\run(\pi, \A)$. An agent strategy $\sigma$ is \emph{winning} in $\G =(\A, F)$ if for every infinite sequence of the input variables $\mathbf{X} \in (2^{\mathcal{X}})^\omega$, the induced play $\trace(\mathbf{X}, \sigma)$ is a winning play with respect to $F$. 

A DFA game with reachability objective $\G = (\A, F)$ can be solved by the following least fixed point computation.
\begin{align*}
    & \Win_0(\G) = F \\
    & \Win_{i+1}(\G) = \Win_i(\G) \cup \mathrm{PreC}(\Win_i(\G))
\end{align*}
where $\mathrm{PreC}(S) = \{ q \in \Q \mid \exists Y \in 2^\mathcal{Y} . \forall X \in 2^\mathcal{X} . \delta(q, (X,Y)) \in S\}$ is the \emph{controllable preimage}. The computation stabilises at the least index $i$ such that $\Win_{i+1}(\G) = \Win_{i}(\G)$ and we denote this fixed point by $\Win(\G) = \Win_{i}(\G)$. Every state $q \in \Win(\G)$ is a winning state for the agent, from where she has a winning strategy in the game $\G_q =(\A_q, F)$, where $\A_q=(2^{\X \cup \Y}, \Q, q, \delta, -)$, i.e.\ the same arena $\A$ but with the new initial state $q$. Intuitively, $\Win(\G)$ represents the ``agent winning region", from which the agent is able to win the game, no matter how the environment behaves. 

\section{Optimal Synthesis}
We consider a set $\Phi$ of \LTLf formulas, of which we intuitively want to satisfy as many as possible.
$\Phi$ is realisable~(resp.~unrealisable) if $\bigwedge_{\varphi \in \Phi} \varphi$ is realisable~(resp.~unrealisable), i.e.\ there exists~(resp.~does not exist) a winning strategy for $\Phi$.
The case we are interested in is where $\Phi$ is unrealisable, but we can realise as large, or important, a subset of $\Phi$ as possible.
When posing this question, we have to define whether this means \emph{first} selecting the subset $\Psi \subseteq \Phi$ we promise to realise and \emph{then} executing the agent with the goal of realising $\Psi$ (guarantee synthesis), or if we find a strategy whose traces satisfy potentially different subsets $\Psi_1,\ldots,\Psi_n \subseteq \Phi$ of $\Phi$, and we maximise the minimal value of the trace (observation synthesis).
We consider both versions as well as a combination, where we maximise a weighted sum of the two values.

Let $\Phi$ be a set of \LTLf formulas over the alphabet $2^{\X \cup \Y}$, where $\X$ and $\Y$ are a fixed partition of the atoms into environment input and agent output, respectively.
Let $\P = (\Phi, \mathcal{X},\mathcal{Y})$ be a tuple and
$G,V: \Phi \rightarrow (0,1]$ be functions specifying the preference value of every \LTLf formula $\varphi \in \Phi$ as guarantee~($G$) and observed value~($V$), respectively. 

We first define the \emph{\maxg} \LTLf synthesis problem. Abusing the notation slightly,
for a strategy $\sigma:(2^\mathcal{X})^*\rightarrow2^\mathcal{Y}$, we define the \emph{guarantee value} of $\sigma$ with respect to $\P$ as \[G(\sigma) = \sum_{\varphi \in \Phi: \sigma \triangleright \varphi} G(\varphi).\]

\begin{definition}[Max-guarantee winning strategy]\label{def:max-guarantee-win}
Given $\P = (\Phi, \mathcal{X},\mathcal{Y})$ and a guarantee value function $G$, 
a strategy $\sigma^*:(2^\mathcal{X})^*\rightarrow2^\mathcal{Y}$ is a \emph{max-guarantee winning strategy} for $\P$ if \[G(\sigma^*) = \max_{\substack{\sigma: (2^\mathcal{X})^*\rightarrow2^\mathcal{Y}}}\ G(\sigma).\]
\end{definition}

\begin{definition}[Max-guarantee \LTLf synthesis]
Given $\P = (\Phi,\mathcal{X},\mathcal{Y})$ and a guarantee value function $G$,
the \emph{max-guarantee synthesis problem} is to compute a pair $(\sigma^\star, v^\star)$, where $\sigma^\star:(2^\mathcal{X})^*\rightarrow 2^\mathcal{Y}$ is a strategy and $v^\star\in\mathbb{R}$ is its guarantee value, such that
\[
v^\star = G(\sigma^\star)
\qquad\text{and}\qquad
v^\star = \max_{\sigma:(2^\mathcal{X})^*\rightarrow 2^\mathcal{Y}} G(\sigma).
\]
\end{definition}

Next, we define the \emph{\maxvalue} \LTLf synthesis problem. 
For a trace $\pi:(2^\mathcal{X\cup Y})^\omega$, we define the \emph{observed value} of $\pi$ with respect to $\P$ as \[V(\pi) = \sum_{\varphi \in \Phi: \pi \models \varphi} V(\varphi),\]
and the observed value of a strategy $\sigma$ as
\[V(\sigma) = \min_{\mathbf{X} \in (2^{\mathcal{X}})^\omega} V(\pi(\mathbf{X},\sigma)).\]

\begin{definition}[\Maxvalue winning strategy]\label{def:max-value-win}
Given $\P = (\Phi, \mathcal{X},\mathcal{Y})$ and an observed value function $V$, 
a strategy $\sigma^*:(2^\mathcal{X})^*\rightarrow2^\mathcal{Y}$ is a \emph{\maxvalue winning strategy} for $\P$ if \[V(\sigma^*) = \max_{\substack{\sigma: (2^\mathcal{X})^*\rightarrow2^\mathcal{Y}}}\ V(\sigma).\]
\end{definition}

\begin{definition}[\Maxvalue \LTLf synthesis]
Given $\P = (\Phi, \mathcal{X},\mathcal{Y})$ and an observed value function $V$, 
the \emph{\maxvalue synthesis problem} is to compute a pair $(\sigma^\star, v^\star)$, where $\sigma^\star:(2^\mathcal{X})^*\rightarrow 2^\mathcal{Y}$ is a strategy and $v^\star\in\mathbb{R}$ is its observed value, such that
\[
v^\star = V(\sigma^\star)
\qquad\text{and}\qquad
v^\star = \max_{\sigma:(2^\mathcal{X})^*\rightarrow 2^\mathcal{Y}} V(\sigma).
\]
\end{definition}

All of these problems are in the same complexity class as ordinary \LTLf synthesis.

\begin{restatable}{theorem}{OptSyntComplexity}\label{thm:opt-synt-complexity}
The max-guarantee and \maxvalue \LTLf synthesis problems are 2EXPTIME-complete.
\end{restatable}
\begin{proof}
    \textbf{Membership.} will be shown in Theorems \ref{thm:correctness} and \ref{thm:complexity}; \ref{thm:max-value} and \ref{thm:max-value-complexity};
    and \ref{thm:incremental} and \ref{thm:incremental-complexity}.
    \noindent\textbf{Hardness.}
     Follows immediately from 2EXPTIME-completeness of \LTLf synthesis itself~\cite{DegVa15}.
\end{proof} 

\begin{example}[Robot example]
The \Cref{robotexample} illustrates how the three notions of optimality lead to different outcomes in the robot example from the introduction.

Max-guarantee favours the path to Room~1, since the switch on the other path prevents any a priori commitment.

Max-observation prefers the path through the switch, since, regardless of the choice of the environment, the robot visits either Rooms~2 and~3 or Rooms~4 and~5.

Incremental max-observation further maximises the ensured value for every history, not only at the initial position. For example, after visiting Room~2, the robot can continue to Room~3 or go down to Rooms~4 and~5 through a door controlled by the environment. From the initial position, both choices ensure visiting two rooms. However, from Room~2, going down to Rooms~4 and~5 ensures visiting three rooms in total, whereas continuing to Room~3 ensures only two rooms.
\end{example}

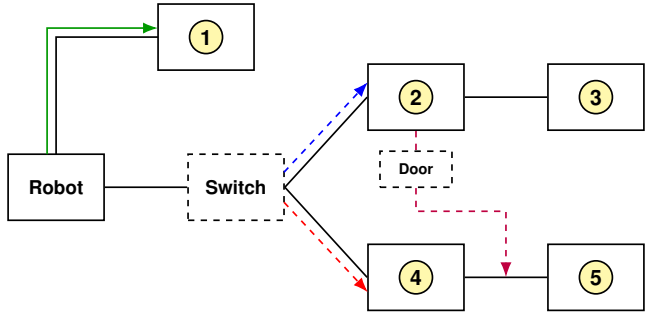
\begin{figure}[t]
  \centering
  \resizebox{\linewidth}{!}{
  \begin{tikzpicture}[
    room/.style={draw, thick, minimum width=1.6cm, minimum height=1.1cm, fill=white, align=center, font=\sffamily\small},
    lbl/.style={circle, draw, thick, fill=yellow!40, inner sep=1pt, minimum size=0.55cm, font=\sffamily\bfseries},
    switch/.style={draw, thick, dashed, minimum width=1.6cm, minimum height=1.1cm, fill=white, align=center, font=\sffamily\small},
    door/.style={draw, thick, dashed, minimum width=1.2cm, minimum height=0.6cm, fill=white, align=center, font=\sffamily\scriptsize},
    arr/.style={-{Latex[length=2.5mm]}, thick},
    arr_blue/.style={-{Latex[length=2.5mm]}, thick, dashed, blue},
    arr_red/.style={-{Latex[length=2.5mm]}, thick, dashed, red},
    arr_green/.style={-{Latex[length=2.5mm]}, thick, green!60!black},
    arr_purple/.style={-{Latex[length=2.5mm]}, thick, dashed, purple},
    txt/.style={font=\sffamily\scriptsize, align=center},
    txt_blue/.style={font=\sffamily\scriptsize, align=center, text=blue},
    txt_red/.style={font=\sffamily\scriptsize, align=center, text=red},
    txt_purple/.style={font=\sffamily\scriptsize, align=center, text=purple}
  ]

  \node[room] (robot) at (0, 3) {\textbf{Robot}};
  \node[switch] (switch) at (3, 3) {\textbf{Switch}};
  \node[room] (r1) at (2.5, 5.5) {};
  \node[lbl] at (r1) {1};
  \node[room] (r2) at (6, 4.5) {};
  \node[lbl] at (r2) {2};
  \node[room] (r3) at (9, 4.5) {};
  \node[lbl] at (r3) {3};
  \node[room] (r4) at (6, 1.5) {};
  \node[lbl] at (r4) {4};
  \node[room] (r5) at (9, 1.5) {};
  \node[lbl] at (r5) {5};

  \draw[thick] (robot.east) -- (switch.west);
  \draw[thick] (0, 3.55) |- (1.7, 5.5);
  \draw[arr_green] (-0.15, 3.55) |- (1.7, 5.65);
  \draw[thick] (switch.east) -- (r2.west);
  \draw[arr_blue] (3.8, 3.25) -- (5.2, 4.75);
  
  \draw[thick] (switch.east) -- (r4.west);
  \draw[arr_red] (3.8, 2.75) -- (5.2, 1.25);

  \draw[thick] (r2.east) -- (r3.west);
  \draw[thick] (r4.east) -- (r5.west);

  \node[door] (door) at (6, 3.3) {\textbf{Door}};
  \draw[thick, dashed, purple] (r2.south) -- (door.north);
  \draw[arr_purple] (door.south) -- (6, 2.6) -| (7.5, 1.5);

  \end{tikzpicture}
  }
  \caption{Robot example.}
  \label{robotexample}
\end{figure}

\section{\MaxG Synthesis}\label{sec:max-g-synt}

In this section, we present a solution to the \maxg\ synthesis problem for an instance $\P = (\Phi,\mathcal{X},\mathcal{Y})$ with a guarantee value function~$G$. 
We show that the problem can be reduced to computing a strategy that realises a maximal subset of formulas, termed a \emph{max realisable core}, such that every play induced by the strategy satisfies all formulas in this subset.

\begin{definition}[Realisable core]
    Let $\Phi = \{\varphi_1, \varphi_2, \cdots \varphi_n\}$ be a set of \LTLf formulas and $\Psi \subseteq \Phi$ a realisable subset of $\Phi$. $\Psi$ is a realisable core if for every $\varphi_i \in \Phi \setminus \Psi$, $\varphi_i \wedge \bigwedge_{\varphi_j \in \Psi}\varphi_j$ is unrealisable.
\end{definition} 

\begin{definition}[Max realisable core]
    A realisable core $\Psi^* \subseteq \Phi$ is \emph{maximal} if for any other realisable core $\Psi' \subseteq \Phi$, it holds that 
    $\sum_{\varphi \in \Psi^*} G(\varphi) \ge \sum_{\varphi \in \Psi'} G(\varphi) $.
\end{definition}

It is straightforward to see that the \maxg\ synthesis problem, computing the maximal guarantee~$v^*$, can be reduced to computing a maximal realisable core~$\Psi^*$, where
$
v^* = \sum_{\varphi \in \Psi^*} G(\varphi).
$
We solve the \maxg synthesis problem in the following steps:

\paragraph{Step 1: Automata Construction.} 
For every \LTLf formula $\varphi_i \in \Phi$, we build its corresponding DFA $\A^i = (2^{ \X \cup \Y}, \Q^i, q_0^i, \delta^i, F^i)$, where $F^i \subseteq \mathcal{Q}^i$ is the set of accepting states for $\varphi_i$. 

\paragraph{Step 2: Game Arena Construction.}
We build the product automaton $\A=\bigotimes_{\varphi_i\in\Phi} \A^i$, such that $\A = (2^{\X\cup\Y}, \Q, q_0, \delta, \emptyset)$, where $\Q=\times_{\varphi_i\in\Phi}\Q^i$ is the product state space, $q_0=(q_0^i)_{\varphi_i\in\Phi}$ is the initial state, $\delta$ is the product transition function induced by the $\{\delta^i\}_{\varphi_i\in\Phi}$.

\paragraph{Step 3: Guarantee Value-Based Partition.} 
     We partition the search space into a sequence of formula subsets ordered by non-increasing guarantee values.
     For each $\varphi_i \in \Phi$, we define $F^i_\otimes = \{(q_1,\cdots,q_n)\in \mathcal{Q} \mid q_i \in F^i\}$ as the set of states in the product automaton where $\varphi_i$ is among the satisfied specifications.  
     We then consider the finite set of values $v_1 > v_2 > \cdots > v_\ell$ such
     that each $v_j = \sum_{\varphi \in \Psi} G(\varphi)$ 
     can take for subsets $\Psi$ of $\Phi$, and define $F_{\Psi} = \bigcap_{\varphi_i\in\Psi} F^i_\otimes$.
     
\paragraph{Step 4: Game Solving.} 
    We iterate over the value $v$ from Step~3 in non-increasing order. 
    For each $\Psi$, we construct the DFA game $\mathcal{G}_\Psi = (\A, F_\Psi)$, using the product automaton 
    $\A$ from Step 2 as the game arena and $F_\Psi$ as the set of accepting states.
    Then we solve the game $\mathcal{G}_\Psi$ by computing its winning set $\mathsf{Win}(\mathcal{G}_\Psi)$.  The max-value $v^*$ is the first $v$ that satisfies $q_0 \in \mathsf{Win}(\mathcal{G}_\Psi)$; we return a winning strategy of $\G_{\Psi}$ where $\Psi$ is a max realisable core. \hfill \qedsymbol

The overall synthesis procedure is presented in \Cref{alg:max_core}. 
The algorithm iterates through $F_{\Psi}$ in non-increasing order of guaranteed values (Lines 2-4) to ensure that the first realisable core found is a maximal one.
Once a realisable core $\Psi$ is found via the reachability check (Lines 7-8), the corresponding guaranteed value with a winning strategy is returned. 

\begin{algorithm}
\caption{\MaxG Synthesis}
\label{alg:max_core}
\begin{algorithmic}[1]
\REQUIRE A set of \LTLf formulas $\Phi = \{\varphi_1,\cdots,\varphi_n\}$
\ENSURE A winning strategy $\sigma^*$ for a max realisable core $\Psi \subseteq \Phi$, and $\sum_{\varphi \in \Psi} G(\varphi)$
\STATE Construct the product automaton $\A$ from $\Phi$
\STATE Partition the search space into a sequence of subsets $\mathcal{S} =  ({\Psi_1}, {\Psi_2}, \dots, \Psi_{2^{|\Phi|}})$ where $\sum_{\varphi \in \Psi_{i}} G(\varphi) \ge \sum_{\varphi \in \Psi_{i+1}} G(\varphi)$ holds for all $i<2^{|\Phi|}$.
\FOR{$j$ in $1$ to $2^{|\Phi|}$}
    \STATE $F_{\Psi_j} \gets \bigcap_{\varphi_i \in \Psi_j} F^i_\otimes$
    \STATE Construct DFA game $\G_{\Psi_j} = (\A, F_{\Psi_j})$
    \STATE Compute winning set $\mathsf{Win}(\mathcal{G}_{\Psi_j})$ and extract a winning strategy $\sigma^*$
    \IF{$q_0 \in \mathsf{Win}(\mathcal{G}_{\Psi_j})$}
        \RETURN $(\sigma^*, \sum_{\varphi \in \Psi_j} G(\varphi))$
    \ENDIF
\ENDFOR
\end{algorithmic}
\end{algorithm}

The correctness of Algorithm~\ref{alg:max_core} is ensured by the guaranteed value-based non-increasing search strategy.

\begin{restatable}{theorem}{maxgcorreect}
\label{thm:correctness}
Algorithm~\ref{alg:max_core} returns a max-guarantee $v = \sum_{\varphi \in \Psi} G(\varphi)$ where $\Psi$ is a max realisable core. 
That is, the subset $\Psi$ is realisable, and for any other realisable subset $\Psi' \subseteq \Phi$, it holds that $\sum_{\varphi \in \Psi} G(\varphi) \geq \sum_{\varphi \in \Psi'} G(\varphi)$.
\end{restatable}
\begin{proof}
It suffices to show that $\Psi$ is a maximal realisable core.

\noindent\textbf{Soundness:} The algorithm returns a subset $\Psi$ only if $q_0 \in \mathsf{Win}(\mathcal{G}_\Psi)$ (Line 8). By the correctness of the reduction of \LTLf synthesis to DFA games \cite{DegVa15}, this condition guarantees that a winning strategy exists for the specification $\Psi$, i.e.\ $\Psi$ is realisable. 

\noindent \textbf{Maximality:} We prove this by contradiction. Assume the algorithm returns a realisable core $\Psi$ of guarantee value $v$, but there exists a realisable core $\Psi'$ with value $> v$. The algorithm iterates through the subsets in $\mathcal{S}$ in non-increasing order of guaranteed values (Lines 2-3). 
Thus, the subset $\Psi'$ would have been visited in an earlier iteration. 
Since $\Psi'$ is realisable, the check on Line~7 would have succeeded, causing the algorithm to return $\Psi'$ (or another subset of the same guaranteed value) and terminate. This contradicts the return of the subset $\Psi$ with lower guaranteed value.
\end{proof}

\begin{restatable}{theorem}{thmComplexity}
\label{thm:complexity}
The running time of \Cref{alg:max_core} is $O(2^{|\Phi|} \cdot |\mathcal{A}|)$, where $|\mathcal{A}|$ is the size of the product automaton $\A$.
\end{restatable}
\begin{proof}
In the worst case, the algorithm explores all $2^{|\Phi|}$ subsets of $\Phi$. For each subset $\Psi$, constructing the DFA game $\mathcal{G}_\Psi$ and computing its winning set can be performed in linear time in the size of the product automaton. Thus, the time complexity is bounded by $O(2^{|\Phi|} \cdot |\mathcal{A}|)$.
\end{proof}

We note that we build the arenas $\bigotimes_{\varphi_i\in\Psi} \A^i$ individually for each $\Psi$.
Normally, the structures of the $\A^i$ are such that these automata are substantially smaller than $\A$ for $\Psi \neq \Phi$.
In this case the running time for $\Psi = \Phi$ normally dominates.

\section{\MaxValue Synthesis}\label{sec:max-obs-synt}
In this section, we first present a basic algorithm for computing a strategy for \maxvalue\ synthesis. 
We then introduce incremental \maxvalue synthesis, where the synthesis space is updated dynamically with the execution history trace, and show how to adapt the basic algorithm to construct an incremental \maxvalue\ strategy, followed by an improved algorithm.

\subsection{Basic Algorithm}\label{subsec:max-obs-basic}

The basic algorithm for the \maxvalue\ synthesis problem begins with automata construction and game arena construction, which coincide with the first two steps of the \maxg\ synthesis procedure described in \Cref{sec:max-g-synt}.
We describe the remaining steps as follows:

\paragraph{Step 3: Target Construction.} 
For each $\varphi_i \in \Phi$, we define $F^i_\otimes = \{(q_1,\cdots,q_n)\in \mathcal{Q} \mid q_i \in F^i\}$ as the set of states in the product automaton where $\varphi_i$ is among the satisfied specifications.  
We then consider the finite set of values $v_1 > v_2 > \cdots > v_\ell$ 
that $\sum_{\varphi \in \Psi} V(\varphi)$ can take for subsets $\Psi$ of $\Phi$, and define
$F_v = \bigcup \{F_\Psi \mid \sum_{\varphi \in \Psi} V(\varphi) \geq v\}$ for each of these values, where $F_{\Psi} = \bigcap_{\varphi_i \in \Psi} F^i_\otimes$.

\paragraph{Step 4: Game Solving.} 
We iterate over the value $v$ from Step 3 
in descending order.
We construct the DFA game $\mathcal{G}_v = (\A, F_v)$, using the product automaton 
$\A$ from Step 2 as the game arena and $F_v$ as the reachability winning objective.
Then we solve the game $\mathcal{G}_v$ by computing its winning set $\mathsf{Win}(\mathcal{G}_v)$.  The max-value $v^*$ is the first $v$ that satisfies $q_0 \in \mathsf{Win}(\mathcal{G}_v)$; we return a winning strategy of $\G_{v^*}$. \hfill \qedsymbol

The following lemma is key to the correctness of the basic algorithm of the \maxvalue synthesis problem.

\begin{restatable}{lemma}{lemMaxValueGame}\label{lem:max-value-game}
Let $\P = (\Phi, \X, \Y)$ with the observed value function $V$ be a \maxvalue synthesis problem and $\mathsf{Win}({\G_v})$ the agent winning region of the reachability game $\G_v = (\A, F_v)$, computed as above, where $v$ is an arbitrary value. 
    Then there is a strategy $\sigma$ with ensured value $v$, i.e.\ $V(\sigma) \geq v$ iff $q_0 \in \mathsf{Win}(\G_v)$.
\end{restatable}
\begin{proof}
    We prove the lemma in both directions.
    \textbf{($\Leftarrow$)} We need to show that if $q_0 \in \mathsf{Win}(\G_v)$ then there is a strategy $\sigma$ such that, $\forall\mathbf{X}\in (2^{\mathcal X})^\omega\, \forall \rho \in \pi(\mathbf{X},\sigma)\sum_{\varphi \in \Phi: \rho \models \varphi} V(\varphi) \geq v$. By construction, $q_0 \in \mathsf{Win}(\G_v)$ shows that there exists an agent strategy $\sigma$ such that, for every $\mathbf{X} \in (2^{\mathcal{X}})^\omega$, the induced play $\trace(\mathbf{X}, \sigma)$ is a winning play of the game $\G_v = (\A, F_v)$. That is to say, the run $\rho = \run(\pi, \A)$ eventually reaches a state in $F_v$ after finitely many steps. Therefore, there exists $k \geq 0$ and a reached state $q_k \in F_v$. By definition of $F_v$, 
    there exists a $\Psi\subseteq\Phi$ with $q_k \in \bigcap_{\varphi_i\in\Psi} F_i$ such that $\sum_{\varphi \in \Psi} V(\varphi) \geq v$ holds.
    Since $\A$ is the product automaton, reaching $q_k$ means that, for every $i$ with $q_k\in F_i$, the $i$-th component automaton $\A^i$ is in an accepting state after reading the finite prefix $\pi(\mathbf{X},\sigma)^k$.
    This implies
    $\pi(\mathbf{X},\sigma)^k\models \bigwedge_{\varphi \in \Psi}\varphi$ and thus $\pi(\mathbf{X},\sigma)\models \bigwedge_{\varphi \in \Psi}\varphi$ and
    $V(\pi(\mathbf{X},\sigma)) \geq \sum_{\varphi \in \Psi}V(\varphi) \geq v$.
    
    Therefore, for every
    $\mathbf{X}$ there exists a finite prefix that satisfies all formulas whose acceptance components contain $q_k$, and the sum of their values is at least $v$. Hence $\sigma$ guarantees value at least $v$, i.e.\ $V(\sigma)\geq v$.

    \noindent\textbf($\Rightarrow$) Assume there exists a strategy $\sigma$ with $V(\sigma)\geq v$. Then, by definition of observed value, for every input $\mathbf{X}\in(2^\X)^\omega$ there exists $k\geq 0$ such that along the induced run of $\A$ the reached product state $q_k$ satisfies $\sum_{i:\ q_k\in F^i_\otimes} V(\A^i) \geq v$, i.e.\ $V(q_k) \geq v$, and therefore $q_k \in F_v$. Thus, under $\sigma$, every play in the game $\G_v=(\A, F_v)$ reaches $F_v$ in finitely many steps. Hence $\sigma$ is a winning strategy for $\G_v$ from $q_0$, and consequently $q_0\in\mathsf{Win}(\G_v)$.
\end{proof}

\begin{restatable}{theorem}{thmMaxValue}\label{thm:max-value}
    Let $\P = (\Phi, \X, \Y)$ with the observed value function $V$ be a max-observation \LTLf synthesis problem and $v^*$ the first value (in descending order) such that $q_0\in\mathsf{Win}(\G_{v^*})$, where $\G_{v^*}=(\A,F_{v^*})$. Let $\sigma^*$ be any winning strategy for $\G_{v^*}$ returned by Step~4. Then $\sigma^*$ is a \maxvalue winning strategy for $\P$, and the maximal realisable value is $v^*$ (i.e.\ $V(\sigma^*)=v^*$).
\end{restatable}
\begin{proof}
    Since $\sigma^*$ is winning for $\G_{v^*}$ and $q_0\in\mathsf{Win}(\G_{v^*})$, Lemma~\ref{lem:max-value-game} shows that $V(\sigma^*) \geq v^*$. To prove optimality, consider any $v$ 
    such that $v > v^*$. By the choice of $v^*$ as the \emph{first} value in descending order with $q_0\in\mathsf{Win}(\G_{v^*})$, we have $q_0\notin\mathsf{Win}(\G_v)$. Applying Lemma~\ref{lem:max-value-game} again, it follows that there is no strategy $\sigma$ with $V(\sigma) \geq v$. Hence no strategy can guarantee a value strictly larger than $v^*$, i.e.\ for all strategies $\sigma$ we have $V(\sigma) \leq v^*$. Therefore $V(\sigma^*)\geq v^*$ and $V(\sigma^*)\le v^*$, so $V(\sigma^*)=v^*$. Moreover, $\sigma^*$ gives the maximal guaranteed value, and is thus a \maxvalue winning strategy that solves the \maxvalue synthesis problem.
\end{proof}

We conclude by showing that the complexity of the basic algorithm matches that of the \maxvalue\ \LTLf\ synthesis problem.

\begin{restatable}{theorem}{thmMaxValueComplexity}\label{thm:max-value-complexity}
The running time of the basic algorithm is $O(2^{|\Phi|} \cdot |\mathcal{A}|)$, where $|\mathcal{A}|$ is the size of the product automaton $\A$.
\end{restatable}
\begin{proof}
In the worst case, the algorithm explores all values $v_1,\ldots,v_\ell$ and all $2^{|\Phi|}$ subsets of $\Phi$ have different values, i.e.\ $\ell= 2^{|\Phi|}$. For each value $v$, constructing the DFA game $\mathcal{G}_v$ and computing its winning set can be performed in linear time in the size of the product automaton. Thus, the time complexity is bounded by $O(2^{|\Phi|} \cdot |\mathcal{A}|)$.
\end{proof}


We note that this can be improved to $O(|\Phi| \cdot |\mathcal{A}|)$ (or: $O(\ln(\ell) \cdot |\mathcal{A}|))$ by considering the values along an ordered tree: if there are $f>1$ feasible values left, we can next test for the $f/2$ largest among them (rounded up or down), reducing the feasible values to $\lceil f/2\rceil$ and providing $\ln(\ell)$ iterations~(e.g.\ $O(\ln(|\Phi|))$ if all formulas have the same value).

The above basic algorithm solves the \maxvalue synthesis problem. We extend this algorithm in the next subsection by introducing the concept of \emph{incremental \maxvalue synthesis}, where the synthesis space is updated dynamically with the execution history trace.



\subsection{Incremental Max-Observation 
Synthesis}
The goal of the principled algorithm from the previous section is to provide maximal ensured value. However, after a prefix $h \in (2^{AP})^*$, we might be in a position to provide \emph{better} ensured values.
We first define the set of \emph{prefixes} of a strategy $\sigma$, denoted $\mathsf{pre}(\sigma) \subseteq (2^{AP})^*$, the set of prefixes of $\{\pi(\mathbf{X},\sigma) \mid \mathbf{X}\in (2^{\mathcal{X}})^\omega\}$ of runs induced by $\sigma$.
Conversely, we define strategies compatible with a history $h \in (2^{AP})^*$ as $\mathsf{comp}(h)=\{\sigma \in (2^\mathcal{X})^*\rightarrow2^\mathcal{Y} \mid h \in \mathsf{pre}(\sigma)\}$.
Then for all $h\in \mathsf{pre}(\sigma)$, we define 
$$V_h(\sigma)= \min \{V(h\cdot\rho) \mid h\cdot \rho \in \pi(\mathbf{X},\sigma),\ \mathbf{X}\in (2^{\mathcal X})^\omega\}.$$

\begin{definition}[Incremental \maxvalue optimal strategy]\label{def:inc-max-value-win}
Given $\P = (\Phi, \mathcal{X},\mathcal{Y})$ with an observed value function $V$, a strategy $\sigma^*:(2^\mathcal{X})^*\rightarrow2^\mathcal{Y}$ is \emph{observation optimal} for a history $h\in \mathsf{pre}(\sigma^*)$ for $\P$ if \[V_h(\sigma^*) = \max_{\sigma\in\mathsf{comp}(h)}V_h(\sigma)\]
holds, and \emph{incremental observation optimal} if it is observation optimal for all histories $h\in \mathsf{pre}(\sigma)$ for $\P$.
\end{definition}

\paragraph{\textbf{Extending the Basic Algorithm.}}
To synthesise an incremental observationally optimal strategy, we can adapt the basic algorithm from \Cref{subsec:max-obs-basic}. 
In particular, in Step~4, after solving each DFA game~$\mathcal{G}_v$, we record for every game state~$q$ the maximal value~$v_q$ such that $q \in \mathsf{Win}(\mathcal{G}_{v_q})$, together with the corresponding local strategy $\sigma_q \in 2^{\mathcal{Y}}$. 
This local strategy can be viewed as the first move of an optimal strategy for the game~$\mathcal{G}_{v_q}$ when the initial game state is replaced by~$q$. 
Finally, we combine the local strategies~$\sigma_q$ for all game states~$q$ into a global strategy $\sigma \colon (2^{\mathcal{X}})^* \to 2^{\mathcal{Y}}$ and show that $\sigma$ is an incremental observationally optimal strategy.

The correctness of this procedure follows from the fact that whenever a game state~$q$ of the product automaton~$\A$ is reached under the strategy~$\sigma$, it follows the corresponding local strategy~$\sigma_q$, thus achieving the maximal observed value~$v_q$ when~$q$ is treated as the initial state of the game.

\paragraph{\textbf{Improved Algorithm.}}
While the basic algorithm can be adapted to compute an incremental max-observation optimal strategy, it is theoretically \emph{inefficient}, as each game state and state-action pair may be considered multiple times: a state $q \in \mathsf{Win}(\mathcal{G}_{v_i})$ also belongs to $\mathsf{Win}(\mathcal{G}_{v_j})$ for all $j > i$.  
To synthesise an incremental max-observation optimal strategy efficiently, we retain the first two steps of the algorithm and then perform incremental game solving and target construction directly.

\paragraph{Step 3: Incremental Max-Obs Optimal Strategy Computation.} 
    We set $\mathsf W_0 = \emptyset$ and then iteratively calculate targets of increasing size as follows for $k=1,\ldots,\ell$:
    \begin{itemize}
        \item 
    $F_{v_k}^= = \bigcup \{ F_\Psi \mid \sum_{\varphi \in \Psi} V(\varphi) = v_k\}$
    \item $T_k = W_{k-1} \cup F_{v_k}^=$,
     $\G_k = (\A,T_k)$, and
    $W_k = \mathsf{Win}(\G_k)$
    \end{itemize}
until $q_0 \in W_k$.

    We use the attractor strategies to each target $T_k$ to define $\sigma^*$~(and fix it arbitrarily elsewhere). \hfill \qedsymbol

    

The correctness of the improved algorithm is based on two facts:
(1) $W_k = \mathsf{Win}(\G_{v_k})$ where $\G_{v_k} = (\A, F_{v_k})$ is the DFA game constructed in Step~4 of \Cref{subsec:max-obs-basic}; and 
(2) for any history $h$, if the maximal achievable value has not already been observed along $h$, then the state $q_h$ reached by $h$ must belong to $\mathsf{Win}(G_{v^*})$ to enforce $v^*$.

\begin{restatable}{theorem}{thmInc}\label{thm:incremental}
    Let $\P = (\Phi, \mathcal{X},\mathcal{Y})$ with an observed value function $V$ be an incremental \maxvalue \LTLf synthesis problem and $v_{k^*}$ 
    the first value (in descending order) such that $q_0\in \mathsf{Win}(\G_{k^*})$, and $\sigma^*$ be a strategy returned by our algorithm. Then $\sigma^*$ is incremental observation optimal.
\end{restatable}
\begin{proof}
Let $\G_{v_k} = (\A, F_{v_k})$ be the DFA game constructed in Step~4 of \cref{sec:max-obs-synt}.
We first show by induction that $W_k = \mathsf{Win}(\G_{v_k})$ is constructed using that all of these games are played on the same arena.

As induction basis, for $k=1$ we have $\G_{v_1}=\G_1$ as $F_{v_1} = F_{v_1}^= = F_{v_1}^= \cup W_0$.

For the induction step, we assume $W_k = \mathsf{Win}(\G_{v_k})$ and show $W_{k+1} = \mathsf{Win}(\G_{v_{k+1}})$ for $k<k^*$.

For this, we use that 
$W_k =\mathsf{Win}(\G_{v_k})\subseteq \mathsf{Win}(\G_{v_{k+1}})$ holds, which is entailed by the larger target  $F_{v_{k+1}} \supseteq F_{v_{k}}$.
At the same time, $W_k =\mathsf{Win}(\G_{v_k})\supseteq F_{v_k}$, as the winning region clearly includes the target, while $F_{v_{k+1}}^= \subseteq F_{v_{k+1}} \subseteq \mathsf{Win}(\G_{v_{k+1}})$ holds by definition and by the winning region including the target, respectively.

Together with $F_{v_{k+1}} = F_{v_{k}} \cup F_{v_{k+1}}^=$ and $T_{k+1} = W_k \cup F_{v_{k+1}}^=$, this provides $F_{v_{k+1}} \subseteq T_{k+1} \subseteq \mathsf{Win}(\G_{v_{k+1}})$.

By monotonicity we get $\mathsf{Win}(\G_{v_{k+1}}) \subseteq \mathsf{Win}(\G_{k+1}) \subseteq \mathsf{Win}\big((\A,\mathsf{Win}(\G_{v_{k+1}}))\big) = \mathsf{Win}(\G_{v_{k+1}})$, where the final equation is provided by the fixed point construction of the winning region. 
This provides $W_{k+1}= \mathsf{Win}(\G_{v_{k+1}})$.

Let now $v_h^* = \max_{\sigma\in\mathsf{comp}(h)}V_h(\sigma)$ for some $h\in \mathsf{pre}(\sigma^*)$.
We first note that $v_h^*\geq v_{k^*}$ must hold as the minimisation reaches over fewer runs. 

If $\max_{k<|h|} \sum_{\varphi \in \Phi: h^k \models \varphi}V(\varphi) =  v_h^*$, then there is nothing to show as the value has already been observed.
Otherwise, we must be in a position to enforce a value of at least $v_h^*$ from the state $q_h$ reached on $h$, which (as we have discussed in the proof of Theorem \ref{thm:max-value}) requires $q_h$ to be in $\mathsf{Win}(\G_{v_h^*})$.
\end{proof}

While the improved algorithm works on the same arena as the algorithm we described in the previous section, it never considers a state-action pair twice.
When explicitly constructing the arena $\A$, the sets $F_{v_k}^=$ can be built with it. In this case, each state is considered at most twice across all fixed point computations: once when added to one of the $F_{v_k}^=$ and once when added to a winning region during fixed point construction.

\begin{restatable}{theorem}{thmIncComplexity}\label{thm:incremental-complexity}
The running time of the improved algorithm is $O(|\mathcal{A}|)$, where $|\mathcal{A}|$ is the size of the product automaton $\A$.
\end{restatable}

This is in contrast to the principled algorithm from the previous section, where every state and state-action pair can be considered once per fixed point construction.

As an additional computational advantage, computing $\mathsf{Win}(\G_k)$ could take fewer steps than computing $\mathsf{Win}(\G_{v_k})$ since we start with a larger target ($T_k \supseteq F_{v_k}$).

Note that the binary search described after Theorem~\ref{thm:max-value-complexity} can not be simply utilised in the improved algorithm. This is because the improved algorithm requires $W_k = \mathsf{Win}(\G_{v_k})$, relying on the monotonicity $\mathsf{Win}(\G_{v_k})\subseteq\mathsf{Win}(\G_{v_{k+1}})$, to enforce the maximal achievable value from each reached state. Skipping intermediate values would discard this information.

\subsection{Combination with Max-Guarantee}
It is straightforward to combine (incremental) \maxvalue with \maxg: to compute the (incremental) \maxvalue value subject to guaranteeing a given subset $\Gamma \subseteq \Phi$ of $\Phi$, we can simply define $F_{v,\Gamma} = \bigcup \{F_\Psi \mid \sum_{\varphi \in \Psi} V(\varphi) \geq v,\ \Psi \supseteq \Gamma\}$ and $F_{v,\Gamma}^= = \bigcup \{F_\Psi \mid \sum_{\varphi \in \Psi} V(\varphi) = v,\ \Psi \supseteq \Gamma\}$.

We can then find optimal solutions and values $v_\Psi^*$ for (incremental) max-observation for each $\Psi \supseteq \Gamma$ individually for realisable $\Psi$ (note that, for $\Psi\neq \emptyset$, $\Psi$ is realisable iff $v_\Psi^*>0$ holds, so that the realisability of $\Psi$ is obtained as a side product and does not have to be tested individually), and find the best value $v_\Psi^* + \sum_{\varphi \in \Psi}G(\varphi)$ for realisable $\Psi$.

\section{Symbolic Optimal Synthesis}
In this section, we introduce symbolic techniques of our proposed approaches to optimal \LTLf synthesis. The symbolic techniques below are faithful encodings of the explicit constructions, following the symbolic framework of~\cite{ZhuTLPV17}. Their correctness follows from Theorems~\ref{thm:correctness},~\ref{thm:max-value}, and~\ref{thm:incremental}.

\paragraph{Symbolic DFA Games.} We utilise the \emph{symbolic} techniques in~\cite{ZhuTLPV17}, where the symbolic representation of a DFA $\A$ = $(2^{AP}, \Q, q_0, \delta, -)$ is a tuple $\A_s = (\X, \Y, \Z, I, \eta)$, where $\Z$ is a set of variables such that $|\Z| = \lceil \log |\Q| \rceil$, and every state $q \in \Q$ corresponds to an interpretation $Z \in 2^\Z$ over $\Z$. $I$ is a Boolean formula satisfied only by the interpretation of the initial state $q_0$. $\eta \colon 2^\X \times 2^\Y \times 2^\Z \rightarrow 2^\mathcal{Z}$ is a Boolean function representing the transition function such that $\eta = \langle \eta_{z_0}, \cdots , \eta_{z_m} \rangle$, where $\eta_z \colon 2^\X \times 2^\Y \times 2^\Z \rightarrow \{0, 1\}$. Accordingly, we denote the symbolic reachability game as $\G_s = (\A_s, g)$, where $g$ is a Boolean formula over $\Z$ representing the reachability objective.

The game can be solved by a symbolic fixpoint computation over two Boolean formulas $w$ over $\Z$ and $t$ over $\Z \cup \Y$, which represent the agent winning region and winning state-action pairs, respectively. The initial step is $w_0(\Z) = g(\Z)$ and $t_0(\Z, \Y) = g(\Z)$. $t_{i+1}$ and $w_{i+1}$ are constructed as follows:

\begin{align*}
     & t_{i+1}(\Z, \Y) = t_i(\Z, \Y) \lor (\neg w_i(\Z) \land \forall X. w_i(\eta(\Z, \X, \Y))) \\ 
     & w_{i+1}(\Z) = \exists Y. t_i(\Z, \Y)
\end{align*}

The computation terminates when $I \models w_i$. Then, we can use $t_{i+1}$ to compute a winning strategy through the mechanism of Boolean functional synthesis~\cite{FriedTV16}. Details of each step can be found in~\cite{ZhuTLPV17}.

\begin{figure*}[h!]
    \centering
    \subfloat[Runtime: basic max-observation (x-axis) vs. max-guarantee and extended max-observation (y-axis).]{
        \centering
        \includegraphics[width=0.3\textwidth]{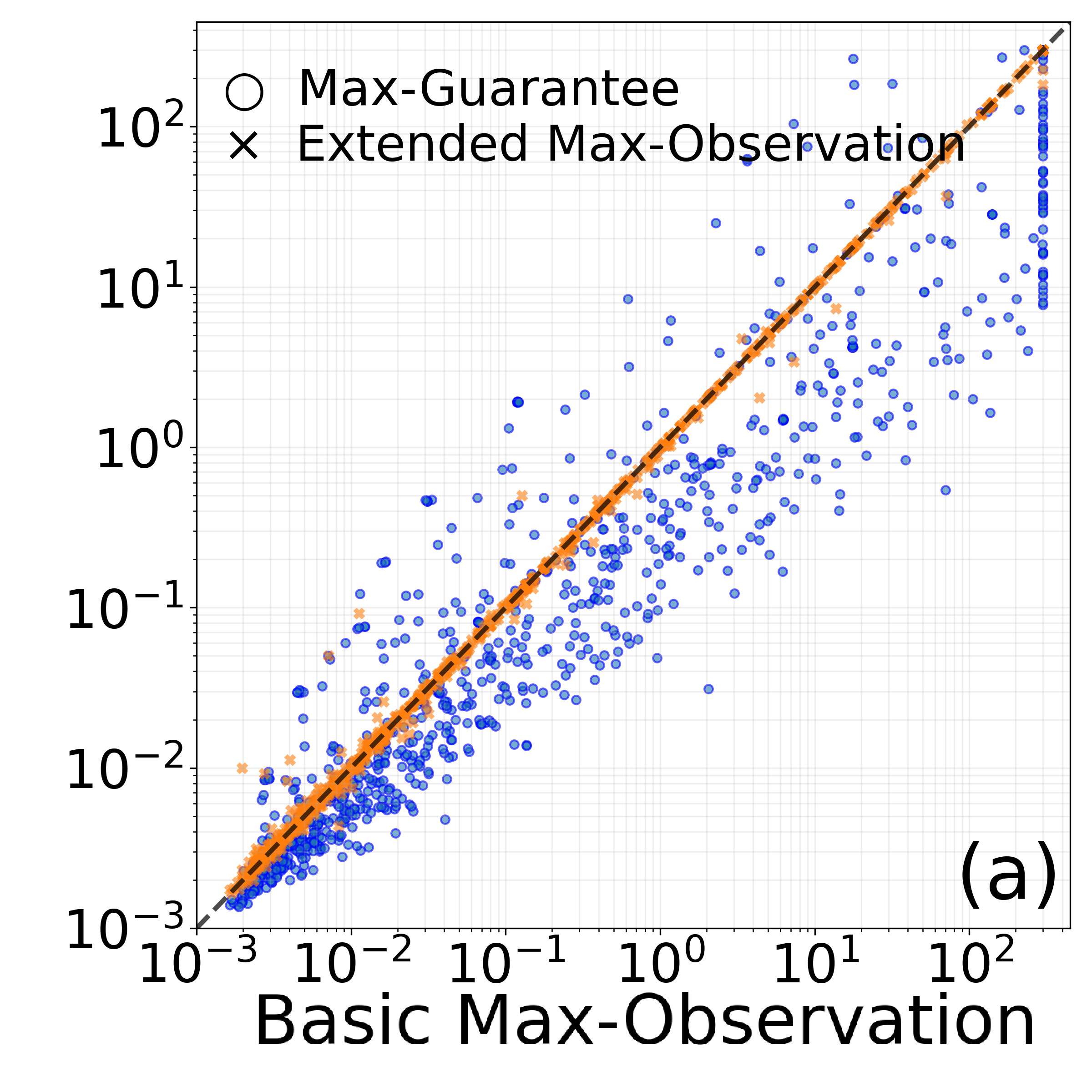}
    	\label{fig:a}
    }
    \hfill
    \subfloat[Maximum BDD size during fixpoint~(preimage) computations in symbolic game solving: basic max-observation vs. max-guarantee.]{
        \centering
       \includegraphics[width=0.3\textwidth]{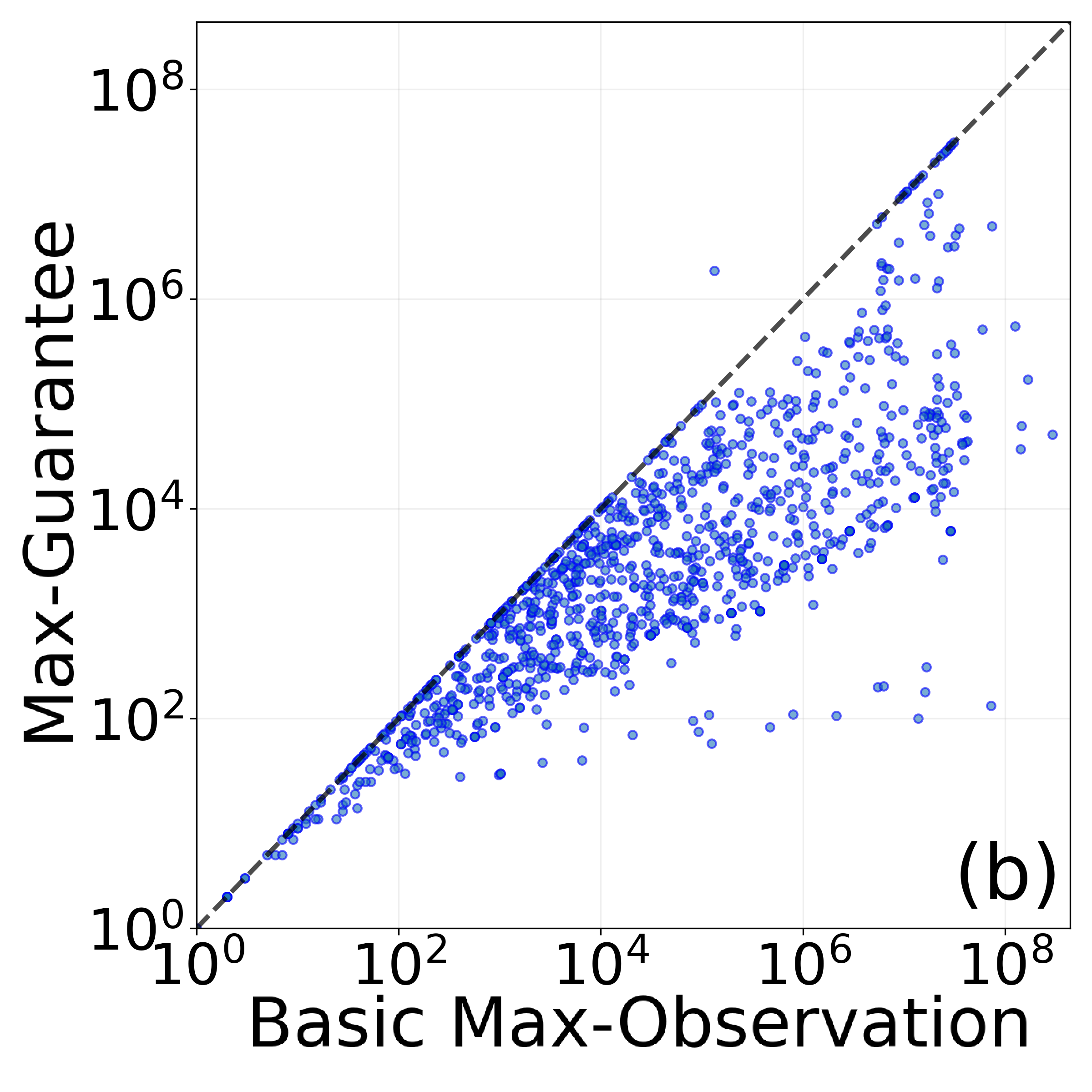}
	  \label{fig:b}
    }
    \hfill
    \subfloat[Runtime: improved max-observation (x-axis) vs. extended max-observation (y-axis).]{
        \centering
        \includegraphics[width=0.3\textwidth]{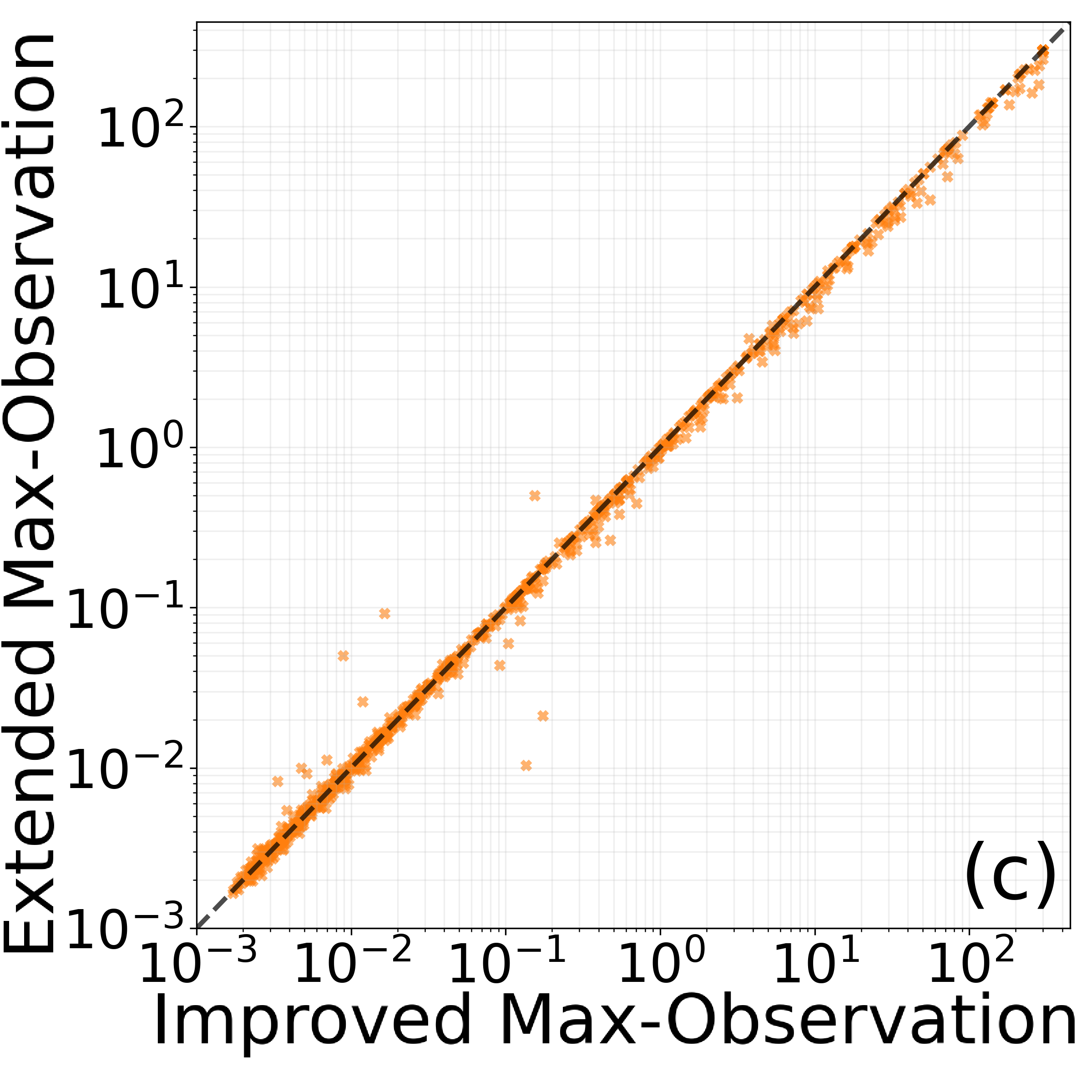}
        \label{fig:c}
    }
    \\[1em]
    \subfloat[Maximum BDD size during fixpoint~(preimage) computations in symbolic game solving: extended max-observation vs. improved max-observation.]{
        \centering
        \includegraphics[width=0.3\textwidth]{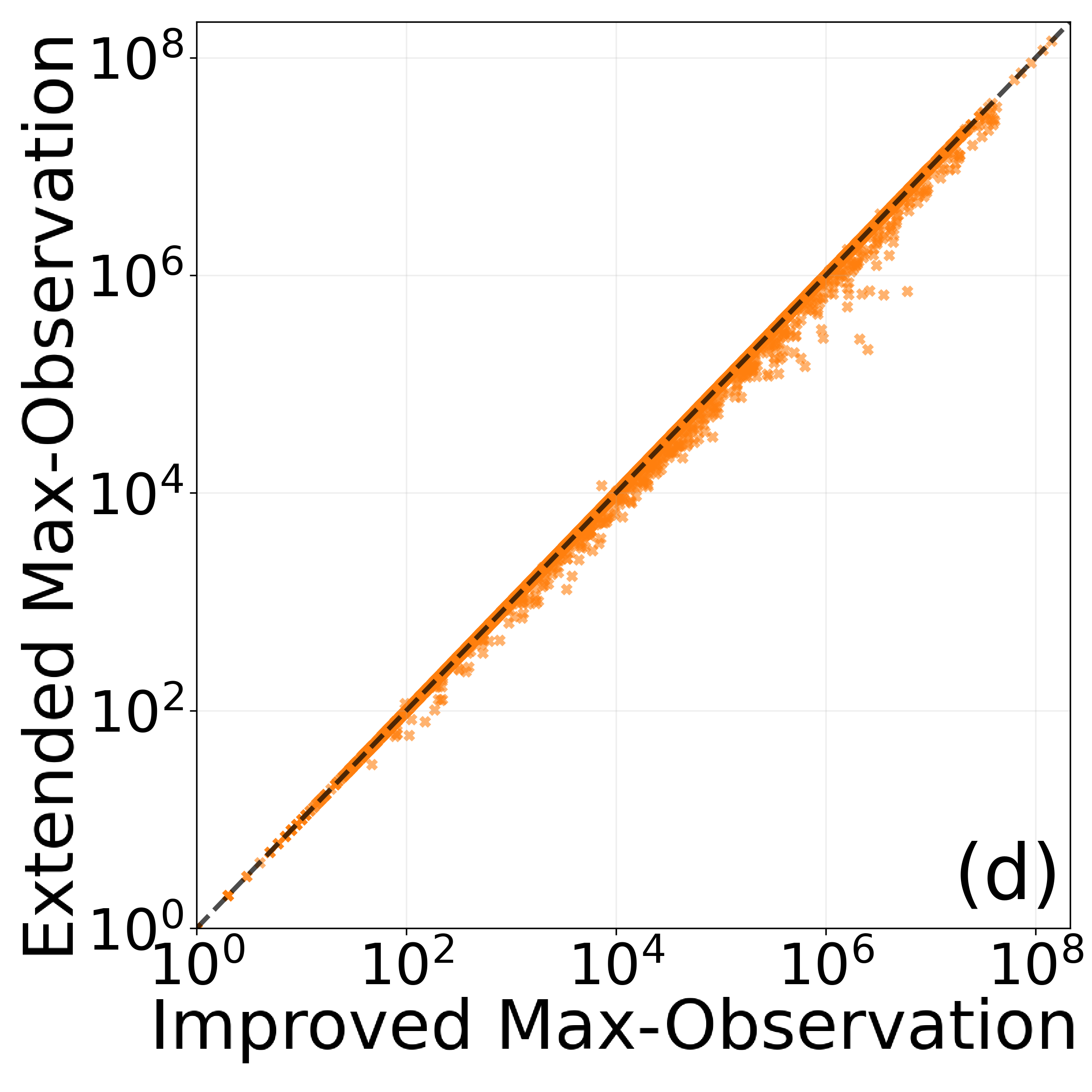}
    	\label{fig:d}
    }
    \hfill
    \subfloat[BDD size of the winning state–action relation $t$: basic max-observation (x-axis) and extended max-observation (y-axis).]{
        \centering
        \includegraphics[width=0.3\textwidth]{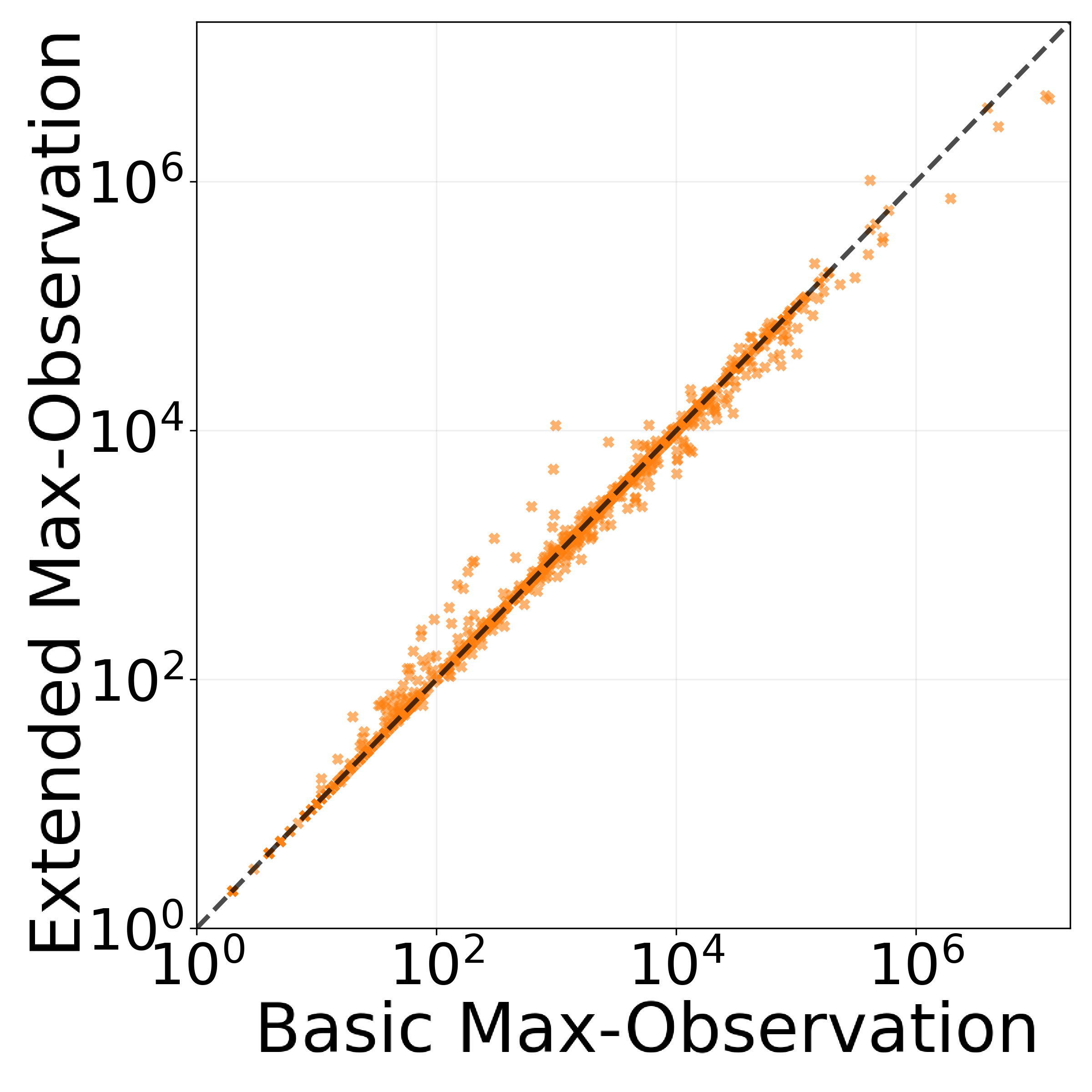}
        \label{fig:e}
    }
    \hfill
    \subfloat[Comparison of the number of preimage computations required to reach the fixpoint between improved max-observation (x-axis) and extended max-observation (y-axis).]{
        \centering
        \includegraphics[width=0.3\textwidth]{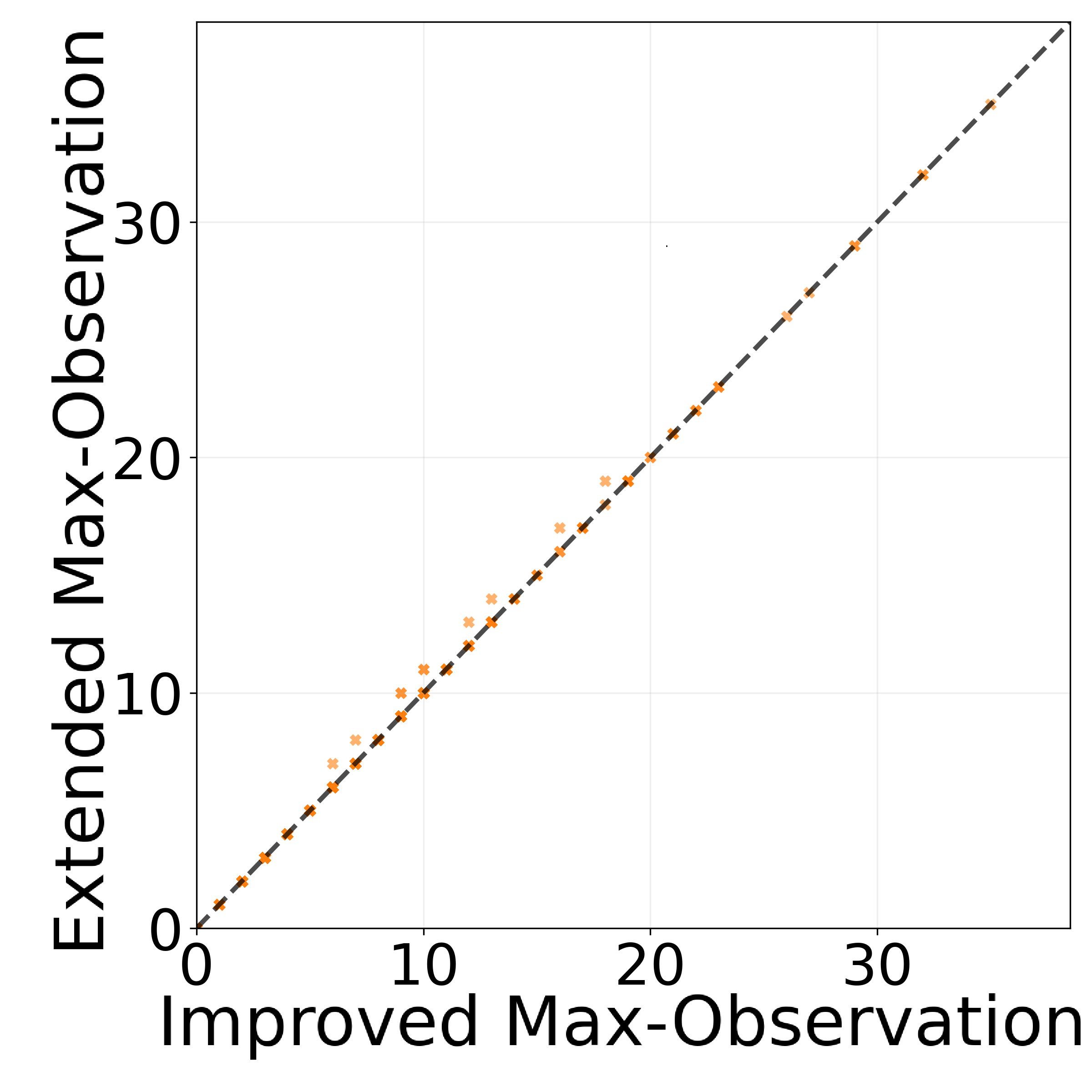}
    	\label{fig:f}
    }
    
    \caption{
    Each point corresponds to one benchmark instance in (a)-(c), one fixpoint step~(preimage computation) in (d)-(e), and the number of fixpoint steps in (f); points below the diagonal indicate better performance of the method on the y-axis.}
\label{fig:evaluation}
\end{figure*}

\subsection{Symbolic Max-Guarantee Synthesis}\label{sec:sym-max-g}

The symbolic characterisation of max-guarantee synthesis from Section~\ref{sec:max-g-synt} is relatively straightforward.
Given a set of \LTLf objectives $\Phi = \{\varphi_1, \cdots, \varphi_n\}$, for each $\varphi_i$ we construct a symbolic DFA $\A_s^i = (\X, \Y, \Z^i, I^i, \eta^i, g^i)$. The symbolic product automaton $\A_s = (\X, \Y, \Z, I, \eta, -)$ can be constructed on the fly by taking  $\Z = \bigcup_i \Z^i$, $I = \bigwedge_i I^i$, and defining the transition function as $\eta(\Z, \X, \Y) = \bigwedge_i \eta^i(\Z^i, \X, \Y)$. Since each $\eta^i = \langle \eta^i_{z_0}, \cdots, \eta^i_{z_m} \rangle$ is represented as a vector of Boolean update formulas over the state variables in $\Z^i$, the transition function $\eta$ of the product automaton can be obtained simply by collecting these update formulas. Then, max-guarantee synthesis reduces to solving a sequence of symbolic DFA games $\G_\Psi = (\A_s, g_\Psi)$, where $g_\Psi = \bigwedge_{\psi_i \in \Psi} g^i$ following the guarantee value-based partition as in Step 4 of \Cref{sec:max-g-synt}.

\subsection{Symbolic \MaxValue Synthesis}\label{sec:sym-max-obs}
We follow the same DFA construction and symbolic product automaton construction as in max-guarantee synthesis. For each value $v$, the reachability objective of the corresponding symbolic game is $g_v = \bigvee_{\Psi \subseteq \Phi: \sum_{\varphi \in \Psi} V(\varphi) \geq v} \bigwedge_{\varphi_i \in \Psi} g^i$. We then solve a sequence of reachability games following the descending order of $v$ as in Step 4 of \Cref{subsec:max-obs-basic}.  

\paragraph{Symbolic Incremental \MaxValue Synthesis.} \label{sec:sym-max-obs-inc}

\emph{Extending the basic max-observation synthesis algorithm} to obtain an incremental max-observation optimal strategy relies on recording and combining winning information across values. Note that solving each game $\G_v$ produces two Boolean formulas: $w_v(\Z)$ encoding the winning region $\Win(\G_v)$ and $t_v(\Z, \Y)$, encoding winning state-action pairs. Instead of explicitly associating each state $q$ with its maximal achievable value, we collect the pairs $(w_v, t_v)$ obtained in descending order of $v$ and construct an incremental observationally optimal strategy by prioritising higher values. We start from the winning state-action pairs of the highest value $v_1$ and iteratively add new winning state-action pairs from lower-value games only for states not already winning at a higher value, i.e.\ $t = t_{v_1} \vee \bigvee_{1 < k \leq v^*} (t_{v_k} \wedge \neg w_{v_{k-1}})$. The resulting formula $t$ thus compactly represents incremental max-observation optimal strategies, and we then can utilise Boolean functional synthesis to obtain a single strategy from $t$.

Symbolic \emph{improved incremental max-observation algorithm} performs incremental game solving directly on growing target sets. Each target set $T_k$ is represented by a Boolean formula $g_{T_k} = w_{k-1} \lor g_{v_k}^=$, where $g_{v_k}^=$ encodes the states in $F_{v_k}^=$ and $w_{k-1}$ is the winning region computed at the previous iteration. Since $t_k$ now encode local strategies for all winning states $w_k$, incremental max-observation optimal strategies can be obtained by simply combining the collected winning state-action pairs $t_k$ from all iterations as $t = t_1 \vee \bigvee_{k > 1} (t_k \wedge \neg w_{{k-1}})$. Note that, although the construction of $t$ is syntactically identical to that used in the extension of the basic max-observation algorithm, the underlying reason differs. In the extension, $t_{v_k} \subseteq t_{v_{k+1}}$ holds due to repeated solving on nested targets, whereas this inclusion does not hold in the improved algorithm, since each iteration already includes the previous winning region in its target. Therefore, the term $(t_k \wedge \neg w_{k-1})$ is necessary to discard vacuous $\mathit{tt}$ agent actions introduced at the initial fixpoint step of game $\G_k$. The final strategy $\sigma^*$ is then synthesised from $t$ using Boolean functional synthesis.

\section{Experimental Evaluation}

\noindent\emph{Implementation.} We implemented a prototype tool, called \tool, to evaluate our approaches. \tool uses Spot~\cite{DuretLutzZPGV25} for state-of-the-art \LTLf-to-DFA translation, and relies on the CUDD~3.0.0~\cite{cudd} BDD library for symbolic synthesis, including Boolean formula representation and manipulation.

\noindent\emph{Benchmarks.} We evaluate our approach on benchmarks taken from the \LTLf synthesis track of SyntCOMP (\url{https://www.syntcomp.org/}). We restrict our evaluation to instances that are unrealisable and conjunction-decomposable, i.e.\ specifications given as a conjunction of multiple individual \LTLf formulas, since our goal is to synthesise strategies that satisfy as many individual specifications as possible. Therefore, we have 1145 benchmark instances, in total. The number of conjuncts per instance ranges from~2 to~50. Source code and benchmarks are available in the supplementary material.

\noindent\emph{Experiment setup.} All tests were run on a MacBook with an Apple M4 chip (10 physical cores, up to 4.41 GHz) and 16 GB of RAM. The timeout was set to 300 seconds.

\noindent\emph{Evaluation results} are shown in \Cref{fig:evaluation}. Generally, different methods work broadly equally fast, with a relatively small factor between them and different approaches perform faster on different instances. We have four approaches: \emph{max-guarantee}, \emph{basic max-observation}, \emph{extended max-observation}, and \emph{improved max-observation}~(the latter two are for incremental max-observation synthesis). Within the 300s timeout, these approaches can solve 1060, 1005, 1006, 1008 instances out of 1145 benchmarks, respectively, with 30 instances timing out during individual \LTLf-to-DFA construction, demonstrating the practical feasibility of our optimal \LTLf synthesis techniques.

Specifically, max-observation achieves a strictly larger ensured value than that returned by max-guarantee synthesis on 16 instances.
This confirms that optimising achieved value rather than a maximal realisable core can sometimes lead to realising more potential.
As shown in \Cref{fig:evaluation}(a), max-guarantee is palpably faster than max-observation on a noticeable subset of instances.
We believe that this is because subsets of specifications induce much smaller symbolic arenas due to the on-the-fly symbolic automata product.
Max-observation, on the other hand, has to always work on the full arena (but can generally realise more specifications and needs fewer calls), hence dealing with more complex game solving.
Moreover, the target sets also become complex.
In particular, timeouts occurred in cases where the target required satisfying, for example, 16 out of 20 objectives, which results in 4845 possible combinations. 
The corresponding target is a large disjunction of individual target sets, and the resulting BDDs (together with their successively constructed controlled attractors) can be large, making symbolic operations costly, as shown in \Cref{fig:evaluation}(b). 


\Cref{fig:evaluation}(a) also shows that extending basic max-observation to incremental strategies incurs only minor overhead, and in some cases even improves performance. This is again explained by the symbolic nature of the implementation, where strategy extraction relies on BDD-based Boolean functional synthesis. The iterative construction of $t$ (see \Cref{sec:sym-max-obs-inc}) can produce smaller and more structured BDD representations, which in turn accelerates strategy extraction. \Cref{fig:evaluation}(e) confirms this.

\Cref{fig:evaluation}(c) compares the extension of basic max-observation with the improved max-observation algorithm. We can see that, despite solving slightly fewer instances, the extension can sometimes outperform the improved algorithm in runtime. This is, again, because the improved algorithm works on larger target sets in each game iteration, which can reduce the number of fixpoint iterations~(see \Cref{fig:evaluation}(f)) but makes each fixpoint step more expensive in practice~(see \Cref{fig:evaluation}(d)), despite its stronger theoretical guarantees.


\section{Conclusions}
We have studied optimal synthesis for \LTLf specifications in cases where the conjunction of objectives is unrealisable. We introduced max-guarantee, max-observation, and incremental max-observation synthesis to capture different optimisation goals and to allow strategies to exploit the observed behaviour of the environment instead of always assuming a purely adversarial one. Our experimental evaluation on a prototype implementation shows that optimal \LTLf synthesis has practical potential, with good scalability and only minor overhead. This work is a first step towards addressing unrealisability in \LTLf synthesis. For future work, we plan to investigate on-the-fly techniques to further improve efficiency and extensions to synthesis under partial observability.

\section*{Acknowledgements}
This work was supported by the EPSRC through grants EP/X03688X/1 (TRUSTED) and
EP/X042596/1 (Games for Good).

\newpage
\bibliographystyle{named}
\bibliography{ref}


\end{document}